\documentclass[11pt]{article}

\usepackage[T1]{fontenc}
\usepackage[utf8]{inputenc}
\usepackage{lmodern}
\usepackage[margin=1in]{geometry}
\usepackage{microtype}

\usepackage{amsmath,amssymb,amsthm,mathtools}
\usepackage{bm}

\usepackage{graphicx}
\graphicspath{{figures/}}
\usepackage{booktabs}
\usepackage{array}
\usepackage{longtable}

\usepackage[shortlabels]{enumitem}
\usepackage[numbers,sort&compress]{natbib}
\usepackage{xcolor}
\usepackage[colorlinks=true,linkcolor=blue,citecolor=blue,urlcolor=blue]{hyperref}
\usepackage[capitalise,noabbrev]{cleveref}

\theoremstyle{plain}
\newtheorem{theorem}{Theorem}[section]
\newtheorem{lemma}[theorem]{Lemma}
\newtheorem{proposition}[theorem]{Proposition}

\theoremstyle{definition}

\theoremstyle{remark}

\newcommand{\R}{\mathbb{R}}

\newcommand{\one}{\mathbf{1}}
\newcommand{\PWidth}{\operatorname{PWidth}}
\DeclareMathOperator{\conv}{conv}
\DeclareMathOperator{\cone}{cone}
\DeclareMathOperator{\dist}{dist}

\DeclareMathOperator*{\argmax}{arg\,max}

\title{Pyramidal Width Can Increase Under Vertex Insertion}
\author{Jinze Zhao\\University of California, San Diego\\\texttt{jiz419@ucsd.edu}}
\date{}

\hypersetup{
  pdftitle={Pyramidal Width Can Increase Under Vertex Insertion},
  pdfauthor={Jinze Zhao}
}

\begin{document}

\maketitle

\begin{abstract}
Lacoste-Julien and Jaggi \cite{lacostejulienjaggi2015} conjectured in 2015 that the pyramidal width of a
polytope cannot increase when a vertex is added, provided that every old point
remains a vertex. We give an exact counterexample with six integer points in
$\R^3$. For
\[
P=\conv\{v_0,\ldots,v_4\},\qquad
Q=\conv\{v_0,\ldots,v_5\},
\]
where
\[
\begin{aligned}
v_0&=(-1,-3,-1), & v_1&=(3,2,-2), & v_2&=(0,2,1),\\
v_3&=(-1,-3,3),  & v_4&=(-2,0,1), & v_5&=(-1,0,-2),
\end{aligned}
\]
all five vertices of $P$ remain vertices of $Q$, but
\[
\PWidth(P)^2=\frac{48}{353}
\quad\text{and}\quad
\PWidth(Q)^2=\frac{36}{133}.
\]
Thus vertex insertion increases pyramidal width by the factor
$\sqrt{1059/532}\approx 1.410886779$. The proof uses the equivalence between
pyramidal width and facial distance, certifies both face lattices by integer
supporting hyperplanes, and evaluates every facial distance by a finite
rational calculation. A dependency-free exact verifier accompanies the paper.


\end{abstract}


\section{Introduction}
\label{sec:introduction}

The Frank--Wolfe method minimizes a differentiable convex function over a
compact convex set using only linear optimization subproblems
\citep{frankwolfe1956,jaggi2013}. For polytopal feasible regions, away-step,
pairwise, and fully corrective variants can converge linearly. The rate bounds
of \citet{lacostejulienjaggi2015} separate analytic conditioning of the
objective from geometric conditioning of the feasible polytope. Their
geometric quantity is the \emph{pyramidal width}.

In Section~3.1, footnote~7 of that work, Lacoste-Julien and Jaggi conjectured
that pyramidal width is non-increasing when another vertex is added, as long
as the old points remain vertices. Such a rule would be useful because it could
transfer lower bounds from a simple containing polytope to a complicated
vertex subset. The conjecture was subsequently invoked as a conditional route
to a pyramidal-width lower bound for combinatorial strategy polytopes
\citep{nakamura2020}.

The conjecture is false. This paper makes three contributions.
\begin{enumerate}[(i)]
  \item We give a six-vertex integer counterexample in dimension three. Adding
  one vertex increases pyramidal width by approximately $41.1\%$.
  \item We give an exact proof. Integer supporting planes certify the complete
  face lattices, and rational arithmetic evaluates the facial distance of every
  nonempty proper face.
  \item We provide a short dependency-free verifier that independently checks
  the combinatorics, all $46$ facial-distance values, and the strict inequality.
\end{enumerate}

The counterexample does not challenge the established convergence theorems
that use pyramidal width as a positive geometric constant. It invalidates only
the proposed general monotonicity principle and deductions that rely on that
principle without additional structure.

The remainder of the paper is organized as follows. \Cref{sec:background}
records the definitions and the facial-distance equivalence of
\citet{pena-rodriguez2019}. \Cref{sec:counterexample} states the counterexample
and certifies its face lattices. \Cref{sec:exact-method} gives the exact finite
distance method, and \cref{sec:evaluation} applies it. Full rational tables and
supporting certificates appear in the appendix.

\section{Pyramidal width and facial distance}
\label{sec:background}

Throughout, all distances and inner products are Euclidean. For a finite set
$A\subset\R^d$ and a nonzero direction $r$, define
\[
\operatorname{dirW}(A,r)
  :=\max_{s,v\in A}\left\langle\frac{r}{\lVert r\rVert},s-v\right\rangle.
\]
Let $M=\conv(A)$. If $x\in M$, let $\mathcal S_x$ be the collection of
subsets $S\subseteq A$ for which $x$ is a proper convex combination of the
points in $S$, and choose
\[
s(A,r)\in\argmax_{v\in A}\langle r,v\rangle.
\]
The pyramidal directional width is
\[
\operatorname{PdirW}(A,r,x)
  :=\min_{S\in\mathcal S_x}
    \operatorname{dirW}\bigl(S\cup\{s(A,r)\},r\bigr).
\]
Following \citet[Section~3]{lacostejulienjaggi2015}, the pyramidal width is
\begin{equation}
\PWidth(A)
  :=\min_{\substack{K\text{ a face of }M,\ x\in K,\\
                    0\ne r\in\cone(K-x)}}
       \operatorname{PdirW}(K\cap A,r,x).
\label{eq:pwidth-definition}
\end{equation}
The minimization over feasible directions in all faces prevents ordinary width
from degenerating along a lower-dimensional face.

For a polytope $R$, write $V(R)$ for its vertex set and abbreviate
$\PWidth(R):=\PWidth(V(R))$. Define its facial distance by
\begin{equation}
\delta(R):=
\min_{\substack{F\text{ a nonempty proper face of }R}}
\dist\!\left(F,\conv\bigl(V(R)\setminus F\bigr)\right).
\label{eq:facial-distance}
\end{equation}
Here $V(R)\setminus F$ means the vertices of $R$ that are not contained in
$F$. Theorems~1 and~2 of \citet{pena-rodriguez2019} imply the exact identity
\begin{equation}
\delta(R)=\PWidth(R).
\label{eq:equivalence}
\end{equation}
This equivalence is the key to the proof: it replaces the nested directional
optimization in \cref{eq:pwidth-definition} by finitely many distances between
convex hulls.

The vertex-insertion conjecture of
\citet[Section~3.1, footnote~7]{lacostejulienjaggi2015}
can now be written as follows. If $A=V(\conv A)$ and all points of $A$ remain
vertices of $\conv(A\cup\{v\})$, then
\begin{equation}
\PWidth(A\cup\{v\})\leq \PWidth(A).
\label{eq:conjecture}
\end{equation}
We disprove \cref{eq:conjecture} in dimension three.

\section{The counterexample and its combinatorics}
\label{sec:counterexample}

Consider the six points in \cref{tab:vertices} and the nested polytopes
\begin{equation}
P:=\conv\{v_0,v_1,v_2,v_3,v_4\},\qquad
Q:=\conv\{v_0,v_1,v_2,v_3,v_4,v_5\}.
\label{eq:PQ}
\end{equation}

\begin{table}[htbp]
  \centering
  \caption{Integer vertices of the counterexample.}
  \label{tab:vertices}
  \begin{tabular}{@{}crrr@{}}
    \toprule
    Vertex & First coordinate & Second coordinate & Third coordinate \\
    \midrule
    $v_0$ & $-1$ & $-3$ & $-1$ \\
    $v_1$ & $ 3$ & $ 2$ & $-2$ \\
    $v_2$ & $ 0$ & $ 2$ & $ 1$ \\
    $v_3$ & $-1$ & $-3$ & $ 3$ \\
    $v_4$ & $-2$ & $ 0$ & $ 1$ \\
    $v_5$ & $-1$ & $ 0$ & $-2$ \\
    \bottomrule
  \end{tabular}
\end{table}

\begin{theorem}[Failure of vertex-insertion monotonicity]
\label{thm:counterexample}
Every point $v_0,\ldots,v_4$ is a vertex of both $P$ and $Q$, while
\begin{equation}
\PWidth(P)^2=\frac{48}{353}
\quad\text{and}\quad
\PWidth(Q)^2=\frac{36}{133}.
\label{eq:main-values}
\end{equation}
Consequently,
\[
\frac{\PWidth(Q)}{\PWidth(P)}
=\sqrt{\frac{1059}{532}}
\approx 1.410886779>1.
\]
\end{theorem}

We first certify the combinatorial part of the theorem. A string such as
$013$ denotes the triangle $\conv\{v_0,v_1,v_3\}$.

\begin{proposition}[Exact face lattices]
\label{prop:face-lattices}
The facets of $P$ are
\begin{equation}
013,\ 014,\ 034,\ 123,\ 124,\ 234,
\label{eq:P-facets}
\end{equation}
and the facets of $Q$ are
\begin{equation}
013,\ 015,\ 034,\ 045,\ 123,\ 125,\ 234,\ 245.
\label{eq:Q-facets}
\end{equation}
In particular, $P$ and $Q$ have $f$-vectors $(5,9,6)$ and $(6,12,8)$,
respectively, and every old vertex remains a vertex after $v_5$ is inserted.
\end{proposition}

\begin{proof}
For each triangle in \cref{eq:P-facets,eq:Q-facets},
\cref{tab:supporting-planes} gives integers $a,b$ such that every relevant
vertex satisfies $a^\top x\leq b$, with equality precisely at the three named
vertices. Thus every listed triangle is a facet.

The $15$ affine determinants in \cref{tab:determinants} are nonzero. Hence the
points are in general affine position, both polytopes are full-dimensional,
and every facet is triangular. Each input point occurs in at least one listed
supporting triangle, so each is extreme. A simplicial three-polytope with
$f_0$ vertices has $f_2=2f_0-4$ facets by Euler's relation
\citep{ziegler1995}. The lists contain exactly $6=2\cdot5-4$ and
$8=2\cdot6-4$ facets, respectively, so no facets are omitted. Their
two-element subsets give the asserted edge counts.
\end{proof}

\begin{figure}[htbp]
  \centering
  \includegraphics[width=\linewidth]{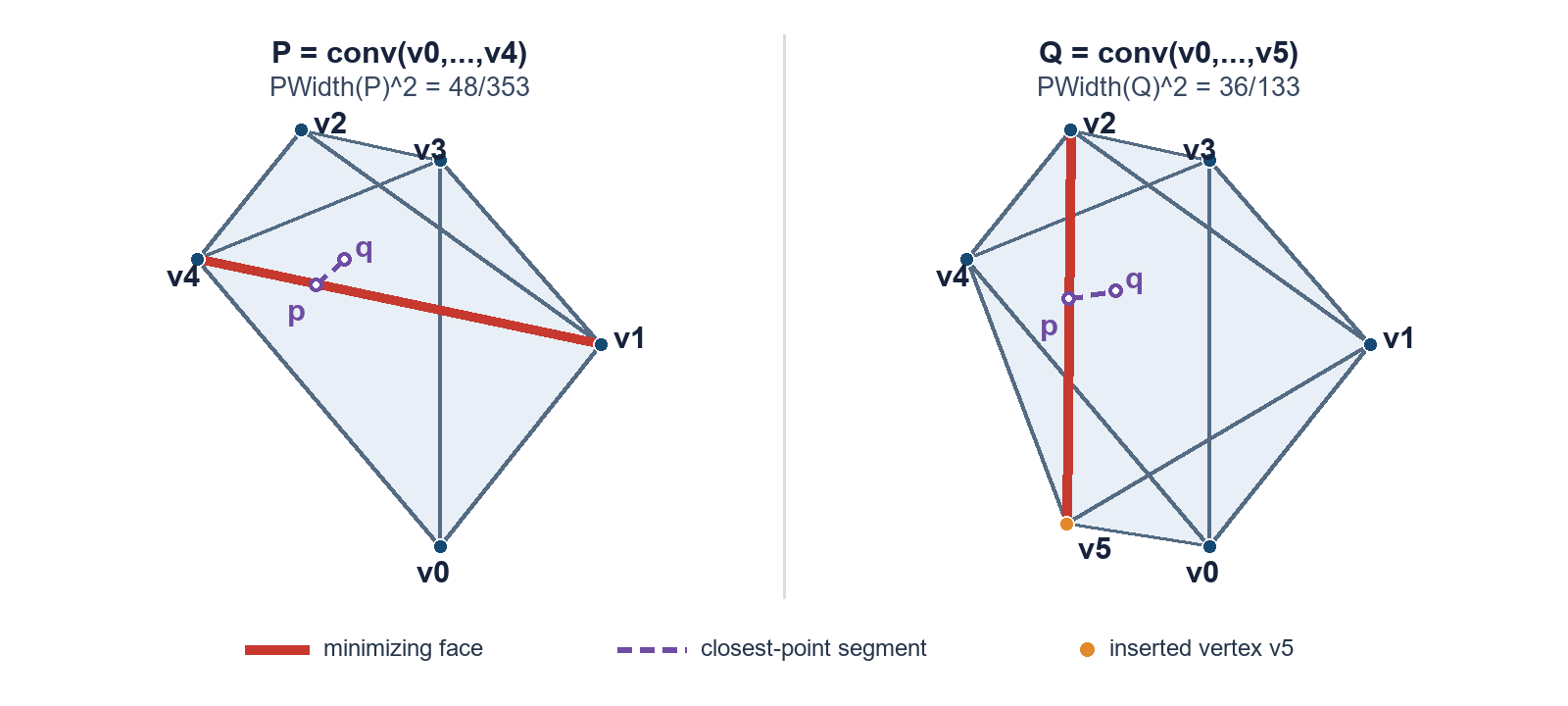}
  \caption{Orthographic visualization of $P$ (left) and $Q$ (right). The
  minimizing face is highlighted in red: edge $14$ for $P$ and edge $25$ for
  $Q$. The dashed segment joins the exact closest points used in
  \cref{sec:evaluation}. The drawing is illustrative and is not used in the
  proof.}
  \label{fig:geometry}
\end{figure}

\section{Exact facial-distance certification}
\label{sec:exact-method}

This section describes a finite rational method for evaluating every term in
\cref{eq:facial-distance}. For finite sets $A,B\subset\R^3$, write
$\conv(A)-\conv(B):=\{x-y:x\in\conv(A),\ y\in\conv(B)\}$ for their
Minkowski difference.

\begin{lemma}[Finite exact distance calculation]
\label{lem:finite-distance}
Let
\[
D(A,B):=\{a-b:a\in A,\ b\in B\}.
\]
Then
\begin{equation}
\dist\bigl(\conv(A),\conv(B)\bigr)^2
=\min_{z\in\conv D(A,B)}\lVert z\rVert^2.
\label{eq:difference-distance}
\end{equation}
If $A$ and $B$ have rational coordinates, the value in
\cref{eq:difference-distance} can be obtained by enumerating affinely
independent subsets of $D(A,B)$ of size at most four and solving rational
linear systems.
\end{lemma}

\begin{proof}
First,
\begin{equation}
\conv(A)-\conv(B)=\conv D(A,B).
\label{eq:minkowski-hull}
\end{equation}
Indeed, if $x=\sum_i\alpha_i a_i$ and $y=\sum_j\beta_j b_j$, then
$x-y=\sum_{i,j}\alpha_i\beta_j(a_i-b_j)$. Conversely, the row and column
marginals of any convex combination of the differences produce points
$x\in\conv(A)$ and $y\in\conv(B)$. Equation~\eqref{eq:difference-distance}
follows immediately.

By Caratheodory's theorem in $\R^3$, a minimizer $z$ has an affinely
independent representation using $k\leq4$ difference vectors
$d_1,\ldots,d_k$. For a proposed support, let $D=[d_1\ \cdots\ d_k]$ and
$G=D^\top D$. The stationary point of $\lVert D\lambda\rVert^2$ on
$\one^\top\lambda=1$ is obtained from
\begin{equation}
\begin{bmatrix}
G&\one\\
\one^\top&0
\end{bmatrix}
\begin{bmatrix}\lambda\\ \mu\end{bmatrix}
=
\begin{bmatrix}0\\1\end{bmatrix}.
\label{eq:kkt-system}
\end{equation}
For an affinely independent support, the bordered matrix is nonsingular. The
candidate is retained when $\lambda\geq0$.

For completeness, singular supports may be skipped because every point of the
convex hull has an affinely independent Caratheodory representation. If the
affine projection for a support has a negative barycentric coordinate, the
minimum over that simplex lies on a proper face, which is covered by a smaller
support. Thus enumeration of all affinely independent supports of size at most
four is exhaustive. Rational inputs make \cref{eq:kkt-system}, its solution,
and the resulting squared norm rational.
\end{proof}

\section{Evaluation of the two polytopes}
\label{sec:evaluation}

For each nonempty proper face $F$ of $P$ and $Q$, apply
\cref{lem:finite-distance} with
\[
A=V(F),\qquad B=V(R)\setminus F.
\]
There are $5+9+6=20$ such faces for $P$ and $6+12+8=26$ for $Q$. The complete
tables are \cref{tab:P-distances,tab:Q-distances}; every entry is a squared
distance.

For $P$, the unique minimum is attained at the edge $F=14$:
\begin{equation}
\delta(P)^2=\frac{48}{353}.
\label{eq:P-minimum}
\end{equation}
The next smallest table entry is $432/233$, so the minimum is strict. Exact
closest points are
\begin{equation}
p=\frac{104v_1+249v_4}{353}
  =\frac{1}{353}(-186,208,41),\qquad
q=\frac{110v_0+243v_2}{353}
  =\frac{1}{353}(-110,156,133).
\label{eq:P-pair}
\end{equation}
Their difference is
\[
r=p-q=\frac{1}{353}(-76,52,-92),
\qquad \lVert r\rVert^2=\frac{48}{353}.
\]
The positive coefficients in \cref{eq:P-pair} place $p$ in edge $14$ and $q$
in $\conv\{v_0,v_2,v_3\}$. Moreover,
\[
\langle r,v_1-p\rangle=\langle r,v_4-p\rangle=0,
\quad
\langle r,v_0-q\rangle=\langle r,v_2-q\rangle=0,
\quad
\langle r,v_3-q\rangle=-\frac{368}{353}<0.
\]
These are the first-order closest-pair inequalities; the exhaustive table
supplies the global lower bound.

For $Q$, the unique minimum is attained at edge $F=25$:
\begin{equation}
\delta(Q)^2=\frac{36}{133}.
\label{eq:Q-minimum}
\end{equation}
The next smallest value is $144/475$. Exact closest points are
\begin{equation}
p'=\frac{76v_2+57v_5}{133}=\frac{1}{7}(-3,8,-2),
\qquad
q'=\frac{49v_1+84v_4}{133}=\frac{1}{19}(-3,14,-2).
\label{eq:Q-pair}
\end{equation}
Thus
\[
r'=p'-q'=\frac{1}{133}(-36,54,-24),
\qquad \lVert r'\rVert^2=\frac{36}{133}.
\]
The active vertices again have zero first-order margins, while the two inactive
complement vertices satisfy
\[
\langle r',v_0-q'\rangle=-\frac{150}{133}<0,
\qquad
\langle r',v_3-q'\rangle=-\frac{246}{133}<0.
\]

Combining \cref{eq:equivalence,eq:P-minimum,eq:Q-minimum} gives
\[
\PWidth(P)^2=\frac{48}{353},\qquad
\PWidth(Q)^2=\frac{36}{133}.
\]
Finally,
\[
\frac{36}{133}-\frac{48}{353}
=\frac{6324}{46949}>0,
\]
which proves \cref{thm:counterexample}.

\section{Consequences, scope, and reproducibility}
\label{sec:discussion}

\paragraph{What the example resolves.}
The construction disproves the general vertex-addition inequality
\cref{eq:conjecture}. In particular, a pyramidal-width lower bound for a vertex
subset cannot be inferred only by inserting its missing vertices into a better
understood containing polytope. The conditional use of that route in
\citet{nakamura2020} therefore requires a separate geometric argument.

\paragraph{What the example does not resolve.}
The example is not a $0/1$-polytope and does not invalidate independently
proved facial-distance bounds for structured families; see, for example,
\citet{chakrabarti2024}. It also does not contradict the positivity of
pyramidal width for a fixed finite vertex set or any Frank--Wolfe convergence
theorem expressed directly in terms of the actual width. A monotonicity result
for adding redundant atoms while keeping the convex hull fixed is a different
statement: here $v_5$ is a new extreme point and changes the polytope.

\paragraph{Reproducibility.}
The accompanying file \texttt{code/verify\_pyramidal\_counterexample.js} shared via the \href{https://drive.google.com/drive/folders/18GYfM3LD8MjNw6QCMuHux525sSFbdLkg?usp=sharing}{link} uses
only Node.js and exact BigInt rational arithmetic. It independently:
\begin{enumerate}[(i)]
  \item enumerates and certifies all supporting facets and affine
  determinants;
  \item reconstructs all vertices, edges, and facets;
  \item enumerates the Caratheodory supports in \cref{lem:finite-distance};
  \item checks every entry of \cref{tab:P-distances,tab:Q-distances}; and
  \item asserts the unique minima and their strict ratio.
\end{enumerate}
Run it from the source directory with
\begin{center}
\texttt{node code/verify\_pyramidal\_counterexample.js}.
\end{center}

\paragraph{Priority.}
To the best of our knowledge, this is the first counterexample to the
vertex-addition monotonicity conjecture of \citet{lacostejulienjaggi2015}.
Targeted exact-phrase, citation, and specialist-literature searches through
July~31, 2026 found no prior proof or counterexample. This statement is a
literature-search report, not a substitute for peer review or a formal priority
determination.

\section{Conclusion}
\label{sec:conclusion}

Pyramidal width is not monotone under vertex insertion, even for
full-dimensional simplicial polytopes in $\R^3$ with integer coordinates. The
explicit pair $P\subset Q$ satisfies
\[
\PWidth(P)^2=\frac{48}{353}<\frac{36}{133}=\PWidth(Q)^2.
\]
The proof is finite and exact: it reduces pyramidal width to facial distance,
certifies every face, and solves only rational quadratic programs of dimension
at most four.

The example leaves useful structured questions open. Monotonicity may still
hold under additional hypotheses on the inserted vertex or for restricted
families, and alternative geometric condition numbers may have better
hereditary behavior. Independently proved bounds for special families, such as
the structured facial-distance bounds of \citet{chakrabarti2024}, are
unaffected. Any lower bound derived solely from the general vertex-insertion
conjecture, however, requires a separate proof.

\section{Disclosure}
\label{sec:disclosure}
The proof strategy and counterexample were produced by OpenAI’s GPT-5.6 Sol Ultra through Codex in
response to prompts from Jinze Zhao. Codex was also used to revise
the exposition and prepare the LaTeX manuscript. The author selected the
problem, directed the interactions and revisions, and is the sole named author. The AI system is acknowledged as a reasoning and writing tool, not
as an author. This disclosure is not a substitute for independent expert
mathematical review.

\bibliographystyle{plainnat}
\bibliography{references}

\appendix
\section{Supporting and affine certificates}
\label{app:combinatorics}

\Cref{tab:supporting-planes} gives primitive integer supporting inequalities
for all facets used in \cref{prop:face-lattices}. Direct substitution verifies
that $a^\top v_i=b$ exactly for the indices in the facet column and that all
other relevant vertices satisfy a strict inequality.

\begin{table}[htbp]
  \centering
  \small
  \caption{Supporting inequalities $a^\top x\leq b$.}
  \label{tab:supporting-planes}
  \begin{tabular}{@{}ccrc@{}}
    \toprule
    Polytope & Facet & \multicolumn{1}{c}{$a$} & $b$ \\
    \midrule
    $P,Q$ & $013$ & $(5,-4,0)$ & $7$ \\
    $P,Q$ & $034$ & $(-3,-1,0)$ & $6$ \\
    $P,Q$ & $123$ & $(5,1,5)$ & $7$ \\
    $P,Q$ & $234$ & $(-1,1,2)$ & $4$ \\
    \midrule
    $P$ & $014$ & $(-13,7,-17)$ & $9$ \\
    $P$ & $124$ & $(-1,1,-1)$ & $1$ \\
    \midrule
    $Q$ & $015$ & $(1,-2,-6)$ & $11$ \\
    $Q$ & $045$ & $(-9,-1,-3)$ & $15$ \\
    $Q$ & $125$ & $(-1,2,-1)$ & $3$ \\
    $Q$ & $245$ & $(-3,3,-1)$ & $5$ \\
    \bottomrule
  \end{tabular}
\end{table}

For $i<j<k<\ell$, define
\[
\Delta_{ijk\ell}
  :=\det\bigl[v_j-v_i\ \ v_k-v_i\ \ v_\ell-v_i\bigr].
\]
The exact values in \cref{tab:determinants} show that no four of the six points
are coplanar.

\begin{table}[htbp]
  \centering
  \small
  \caption{All four-point affine determinants.}
  \label{tab:determinants}
  \begin{tabular}{@{}crcrcr@{}}
    \toprule
    Indices & $\Delta$ & Indices & $\Delta$ & Indices & $\Delta$ \\
    \midrule
    $0123$ & $ 60$ & $0124$ & $-12$ & $0125$ & $-42$ \\
    $0134$ & $-68$ & $0135$ & $-48$ & $0145$ & $-38$ \\
    $0234$ & $-32$ & $0235$ & $-12$ & $0245$ & $-20$ \\
    $0345$ & $-12$ & $1234$ & $-36$ & $1235$ & $-66$ \\
    $1245$ & $-12$ & $1345$ & $ 46$ & $2345$ & $ 28$ \\
    \bottomrule
  \end{tabular}
\end{table}

\clearpage
\section{Complete facial-distance tables}
\label{app:distance-tables}

Each value below is
\[
d_R(F)^2
  :=\dist\!\left(F,\conv\bigl(V(R)\setminus F\bigr)\right)^2.
\]
Faces are written as vertex-index strings. The tables contain every nonempty
proper face from \cref{prop:face-lattices}.

\begin{table}[htbp]
  \centering
  \small
  \caption{All squared facial distances for $P$.}
  \label{tab:P-distances}
  \begin{tabular}{@{}crcr@{}}
    \toprule
    Face $F$ & $d_P(F)^2$ & Face $F$ & $d_P(F)^2$ \\
    \midrule
    \multicolumn{4}{c}{\textit{Vertices}} \\
    $0$ & $4624/483$ & $1$ & $177/10$ \\
    $2$ & $54/19$   & $3$ & $80/7$ \\
    $4$ & $32/13$   &     & \\
    \addlinespace
    \multicolumn{4}{c}{\textit{Edges}} \\
    $01$ & $4624/419$ & $03$ & $10$ \\
    $04$ & $16/5$     & $12$ & $8$ \\
    $13$ & $3600/1619$& $14$ & $\mathbf{48/353}$ \\
    $23$ & $432/233$  & $24$ & $108/17$ \\
    $34$ & $16/5$     &      & \\
    \addlinespace
    \multicolumn{4}{c}{\textit{Facets}} \\
    $013$ & $108/17$   & $014$ & $432/233$ \\
    $034$ & $8$        & $123$ & $16/5$ \\
    $124$ & $10$       & $234$ & $4624/419$ \\
    \bottomrule
  \end{tabular}
\end{table}

\begin{table}[htbp]
  \centering
  \small
  \caption{All squared facial distances for $Q$.}
  \label{tab:Q-distances}
  \begin{tabular}{@{}crcr@{}}
    \toprule
    Face $F$ & $d_Q(F)^2$ & Face $F$ & $d_Q(F)^2$ \\
    \midrule
    \multicolumn{4}{c}{\textit{Vertices}} \\
    $0$ & $72/17$    & $1$ & $108/7$ \\
    $2$ & $54/19$    & $3$ & $80/7$ \\
    $4$ & $784/395$  & $5$ & $1444/507$ \\
    \addlinespace
    \multicolumn{4}{c}{\textit{Edges}} \\
    $01$ & $576/257$   & $03$ & $91/10$ \\
    $04$ & $144/475$   & $05$ & $92/19$ \\
    $12$ & $38/5$      & $13$ & $3600/1619$ \\
    $15$ & $441/101$   & $23$ & $432/233$ \\
    $24$ & $115/34$    & $25$ & $\mathbf{36/133}$ \\
    $34$ & $16/5$      & $45$ & $16/11$ \\
    \addlinespace
    \multicolumn{4}{c}{\textit{Facets}} \\
    $013$ & $1764/467$ & $015$ & $100/11$ \\
    $034$ & $16/3$     & $045$ & $16/5$ \\
    $123$ & $16/5$     & $125$ & $16/3$ \\
    $234$ & $100/11$   & $245$ & $1764/467$ \\
    \bottomrule
  \end{tabular}
\end{table}

\section{Machine-checkable certificate}
\label{app:verifier}

The verifier represents a rational number as a normalized pair of arbitrary
precision integers. Gaussian elimination solves \cref{eq:kkt-system} without
rounding. Its expected terminal summary is
\begin{verbatim}
PWidth(P)^2 = 48/353 at face 14
PWidth(Q)^2 = 36/133 at face 25
squared ratio = 1059/532
All exact dependency-free checks passed, including every table entry.
\end{verbatim}
The script also asserts each intermediate value, so this summary is printed
only after all combinatorial and distance certificates have passed.

\end{document}